\documentclass{article}
\usepackage{ijcai24}

\usepackage{times}
\usepackage{soul}
\usepackage{url}
\usepackage[hidelinks]{hyperref}
\usepackage[utf8]{inputenc}
\usepackage[small]{caption}
\usepackage{graphicx}
\usepackage{amsmath}
\usepackage{amsthm}
\usepackage{booktabs}
\usepackage{algorithm}
\usepackage{algorithmic}
\usepackage[switch]{lineno}

\usepackage{bm}
\usepackage{amssymb}
\newtheorem{theorem}{Theorem}

\usepackage[american]{babel}

\newcommand{\x}{{\bm x}}
\newcommand{\y}{{\bm y}}
\newcommand{\gt}{{\bm p}}
\newcommand{\z}{{\bm z}}

\newcommand{\mC}{{\mathcal C}}
\newcommand{\mX}{{\mathcal X}}
\newcommand{\mD}{{\mathcal X}}
\newcommand{\mY}{{\mathcal Y}}
\newcommand{\mT}{{\mathcal C}}

\newcommand{\mE}{{\mathbb E}}
\renewcommand{\geq}{\geqslant}
\renewcommand{\leq}{\leqslant}

\newcommand{\ie}{\textit{i.e.}}

\title{InfoMatch: Entropy Neural Estimation for Semi-Supervised Image Classification}
\author{
	Qi Han\and Zhibo Tian\and Chengwei Xia\And Kun Zhan\footnote{Corresponding Author.}\\
	\affiliations
	School of Information Science and Engineering, Lanzhou University
	\emails
	kzhan@lzu.edu.cn
}

\begin{document}
	\maketitle
	\begin{abstract}
		Semi-supervised image classification, leveraging pseudo supervision and consistency regularization, has demonstrated remarkable success. However, the ongoing challenge lies in fully exploiting the potential of unlabeled data. To address this, we employ information entropy neural estimation to utilize the potential of unlabeled samples. Inspired by contrastive learning, the entropy is estimated by maximizing a lower bound on mutual information across different augmented views. Moreover, we theoretically analyze that the information entropy of the posterior of an image classifier is approximated by maximizing the likelihood function of the softmax predictions. Guided by these insights, we optimize our model from both perspectives to ensure that the predicted probability distribution closely aligns with the ground-truth distribution. Given the theoretical connection to information entropy, we name our method \textit{InfoMatch}. Through extensive experiments, we show its superior performance. The source code is available at \url{https://github.com/kunzhan/InfoMatch}.
	\end{abstract}
	\section{Introduction}\label{sec:intro}
	Deep learning image classifiers typically depend on a considerable amount of labeled data, encountering performance limitations as the dataset scale increases. In contrast, semi-supervised learning (SSL)~\cite{lee2013pseudo,rasmus2015semi} leverages both limited labeled data and a substantial pool of unlabeled data, achieving performance comparable to or even surpassing fully-supervised methods with fewer labels.
	
	Most existing SSL methods utilize two key strategies: pseudo supervision and consistency regularization. FixMatch~\cite{sohn2020FixMatch} effectively combines these strategies. It utilizes weak augmented view to generate pseudolabels that serve as guidance for the prediction of strong augmented view. FlexMatch~\cite{zhang2021flexmatch}, FreeMatch~\cite{wang2022freematch}, and SoftMatch~\cite{chen2023softmatch} primarily focus on improving the quality and quantity of pseudolabels. These methods center on pseudolabel selection and leverage confidence metrics to assess uncertainty, labeling samples only when confidence surpasses a predefined threshold. Additionally, variants such as MixMatch~\cite{berthelot2019mixmatch} and ReMixMatch~\cite{berthelot2019remixmatch} delve into generating numerous novel samples through diverse data augmentation techniques, intending to introduce noise and transformations for enhanced model robustness.
	
	Our motivation does not focus on pseudolabel selection strategies and data augmentation techniques. Instead, we aim to efficiently exploit the potential of unlabeled data by adapting the entropy neural estimation from unsupervised representation learning~\cite{hjelm2018learning,DSSN2023}. The fundamental principle is to maximize the lower bound of mutual information between two augmented views~\cite{nowozin2016f,belghazi2018mutual,hjelm2018learning,DSSN2023}. These infomax approaches leverage all available unlabeled data, allowing the model to delve deeper into the inherent structure and patterns of the data, ultimately improving classification accuracy. We generate two strong augmentation views for all unlabeled data, utilizing the mutual information between these views as a lower bound on the dataset entropy. This strategy ensures the model consistently produces accurate output under varying input conditions, accurately capturing the most salient features of the data.
	
	Another motivation of this paper is to improve classifier accuracy by approximating the posterior probability. The probability of a sample belonging to a class is determined by the posterior of the classifier. A neural classifier models its last layer using a softmax function as a parameterized posterior, meaning the classifier's predictive probability distribution approximates the ground-truth posterior distribution. We investigate the connection between posterior entropy and the likelihood function of its predictions. Our observation suggests that the upper bound of posterior entropy is effectively approximated by maximizing the softmax prediction's likelihood. This finding is critical for advancing model optimization, ensuring the predicted-probability distribution more accurately reflects the true distribution. We approximate the entropy of the ground-truth posterior. To apply the posterior entropy neural estimator to unlabeled samples, we employ pseudo supervision and weak-to-strong strategies in addition to the supervision loss of labeled data. We generate pseudolabels through a weak augmentation view, guiding the prediction of two strong augmented views. Additionally, we introduce CutMix~\cite{yun2019CutMix} as a strong augmentation view to capture more non-deterministic target features.
	
	In \textit{InfoMatch}, we are motivated by two key objectives: estimating the entropy of the data and estimating the entropy of the ground-truth posterior. We propose a novel entropy-based methodology that integrates a data entropy neural estimator with a posterior entropy neural estimator. This combination prompts the model to thoroughly explore the intrinsic structure defined by entropy, achieved by maximizing the mutual information between augmented views of all unlabeled data. Moreover, we make the crucial observation that the upper bound on posterior entropy is effectively approximated by maximizing the predicted softmax likelihood function. This insightful observation is applied to unlabeled data in our approach through weak-to-strong, pseudo-supervision, and CutMix strategies. By approximating these entropies through gradient descent, our method progressively captures information about unlabeled data and model characteristics. Since our method primarily focuses on information entropy estimations, we aptly name it \textit{InfoMatch}. The extensive experimental results consistently validate the effectiveness of \textit{InfoMatch}, especially in scenarios where labeled data is scarce. Our dual-entropy-based method stands as a robust solution for fully exploiting the potential of unlabeled data, providing valuable insights into both the data structure and the posterior probabilities, thereby enhancing the performance of semi-supervised image classification.
	
	The main contributions are summarized as follows: 1) Our proposed \textit{InfoMatch} effectively exploit the potential of unlabeled data. 2) Leveraging multiple objectives, \textit{InfoMatch} approximates the data entropy and the ground-truth posterior entropy. By utilizing strong-to-strong constrastive, weak-to-strong pseudo supervision, and CutMix strategies, we apply this approach to unlabeled data, enabling more efficient utilization of unlabeled data information. 3) Experiments across various semi-supervised learning benchmarks validate the superior performance of \textit{InfoMatch}. 
	\section{Related Works}
	Semi-supervised learning is a crucial branch in the field of machine learning and computer vision. Its fundamental concept lies in utilizing the data distribution information latent in a substantial quantity of unlabeled samples to enhance learning performance when only a limited number of labeled samples are available. The main strategies behind SSL are pseudo supervision and consistency regularization.
	
	\textbf{Consistency regularization.} Based on the smoothness assumption, consistency regularization suggests that minor perturbations applied to unlabeled samples does not result in substantial variations in their predictions. To ensure consistency, $\Pi$ Model~\cite{rasmus2015semi} introduces perturbations through data augmentation and dropout, aiming to maximize the similarity of predictions derived from two forward propagations of identical unlabeled sample. Temporal Ensembling~\cite{samuli2017temporal} employs a time series combination model to minimize the mean square error loss between current and historical predictions, thus streamlining forward reasoning. MeanTeacher~\cite{tarvainen2017mean} converts the exponential moving average of the prediction results into the model weight, penalizing the difference in predictions between student model and teacher model. In addition, Unsupervised Data Augmentation (UDA)~\cite{xie2020unsupervised} extends the data augmentation method and employs specific target data augmentation algorithms for specific tasks. In contrast, Pseudo Label~\cite{lee2013pseudo} works by generating artificial labels for unlabeled data and performing fully-supervised training using both labeled and pseudo-labeled data.
	
	\textbf{Holistic methods.} Most SSL methods combine pseudo supervision with consistency regularization to improve performance. MixMatch~\cite{berthelot2019mixmatch} leverages unlabeled data by integrating both consistency regularization and entropy minimization. Based on MixMatch, RemixMatch~\cite{berthelot2019remixmatch} introduces two innovative strategies, distribution alignment and augmentation anchor, to enhance its robustness and accuracy. FixMatch~\cite{sohn2020FixMatch} which maintains the use of Augmentation Anchor, simplifies the number of strong augmemtations. To improve the accuracy of pseudolabels, a fixed high threshold is established to eliminate unreliable labels. To alleviate class imbalance and address the issue of low utilization of unlabeled data in the early stage, FlexMatch~\cite{zhang2021flexmatch} and FreeMatch~\cite{wang2022freematch} take the model's learning status and class learning difficulty into consideration to adaptively adjust the confidence threshold. SoftMatch~\cite{chen2023softmatch} uses a truncated Gaussian function to assign weights to samples based on their confidence, addressing the inherent trade-off between the quantity and quality of pseudolabels.
	
	Furthermore, methods such as CoMatch \cite{li2021comatch}, SimMatch \cite{zheng2022simmatch}, and SimMatchV2 \cite{zheng2023simmatchv2} integrate graph-based learning methods into semi-supervised image classification, jointly learning two representations, class probabilities and low-dimensional embeddings, of the training data.
	
	\textbf{Mixing augmentation.} The two prevalent mixing methods for images are mixup~\cite{zhang2018mixup} and CutMix~\cite{yun2019CutMix}. Mixup generates novel training samples by combining two images and their corresponding labels linearly, aiming to facilitate the model in learning smoother decision boundaries. Conversely, CutMix swaps segments of one image with those of another, updating the corresponding labels simultaneously to encourage the model to comprehend and leverage local information within the image. Interpolation consistency training~\cite{verma2022interpolation}, MixMatch~\cite{berthelot2019mixmatch} and RemixMatch~\cite{berthelot2019remixmatch} combine mixup with consistency regularization to expand the dataset by generating virtual ``mixed'' training samples, and utilize consistency regularization to ensure its generalization ability on mixed data. While FMixCutMatch~\cite{wei2021fmixcutmatch} combines Fourier space-based data cutting and data mixing to augment both labeled and unlabeled samples, thereby enhancing its performance and generalization. Furthermore, \cite{ghorban2022individual} utilizes mixup and CutMix simultaneously, aiming to blend the semantic information of three images and introduce enhanced perturbations into the data. To further enhance its attention to local information and improve its sensitivity to diverse position cues, we employ CutMix~\cite{yun2019CutMix} in \textit{InfoMatch}.
	\section{Entropy Neural Estimation}
	We employ an independent and identically distributed (i.i.d.) training dataset denoted as $\mD = \mD_l \cup \mD_u$, where $\mD_l$ comprises labeled data pairs $(\x_i, \gt_i), \forall\,i \in \{1, \ldots, n_l\}$, and $\mD_u$ consists of unlabeled data $\x_i, \forall\,i \in \{n_l+1, \ldots, n_l+n_u\}$. The sizes of labeled and unlabeled datasets are denoted by $n_l$ and $n_u$, respectively, with ${n_l} \ll {n_u}$, and the total dataset size is $n = n_l + n_u$. Here, each $\x_i$ represents a data point, and $\gt_i$ corresponds to its ground-truth class label, represented using a one-hot encoding. $\mC=\{c_1,\ldots,c_k\}$ denotes the class set. If $\x_i$ belongs to $c_j$ class, we have $p(c_j|\x_i)=p_{ij}=1$\,.
	
	We employ an encoder to generate a latent feature $\z_i=[z_{ij}]$ and define the posterior probability $\sigma(c_{j}|\x_i,\theta)$ of $\x_i$ belonging to the $j$-th class $c_j$ using a softmax function $\sigma(\cdot)$:
	\begin{align}
		y_{ij}=\sigma(c_{j}|\x_i,\theta)=\frac{\exp(z_{ij})}{\sum_j \exp(z_{ij})}\,.
		\label{posterior}
	\end{align}
	$y_{ij}$ denotes the probability of $\x_i$ belonging to the $j$-th class while $p_{ij}=p(c_j|\x_i)=1$ representing a one-hot encoding. It means that we model $\sigma(c_{j}|\x_i,\theta)$ as an encoder to approximate the truth posterior $p_{ij}=p(c_j|\x_i)=1$.
	
	According to the second term of the information bottleneck theory~\cite{tishby2000information},
	$\min {\rm I}(\mD;\mY)-\beta{\rm I}(\mY;\mT)$, the neural network codes $\sigma(c_{j}|\x_i,\theta)$ to close to $p(c_j|\x_i)$\,. Here, ${\rm I}(\cdot,\cdot)$ denotes the mutual information, and ${\rm H}(\cdot)$ is the entropy.
	\begin{theorem}\label{the1}
		For an i.i.d. finite dataset $(\mX,\mT)$, the approximation ${\rm H}(\mT|\mX)\simeq-\ln p(\mT|\mX)$ holds\,.
	\end{theorem}
	\begin{proof}
		Since $\x_i$ are sampled i.i.d., $p(c|\x_i)$ corresponds point-to-point with $\x_i$ and is regarded as an i.i.d. set, the posterior is $p(\mT|\mX)=\prod_{i=1}^np(c|\x_i)$\,. Then the expectation with respect to $p(\x)$ is approximated by a finite sum over $\mX$, so that
		\begin{align*}
			{\rm H}(\mT|\mX)
			&=-\mE_{p(\x)}\mE_{p(c|\x)}\ln p(c|\x)\notag\\
			&\simeq-\sum_{i=1}^n\ln p(c|\x_i)=-\ln\prod_{i=1}^n p(c|\x_i)\notag\\
			&=-\ln p(\mT|\mX)\,.
		\end{align*}
		Thus, the equation ${\rm H}(\mT|\mX)\simeq-\ln p(\mT|\mX)$ holds.
	\end{proof}
	\subsection{Posterior Entropy Neural Estimation}
	Suppose that the ground-truth posterior is being generated from the unknown distribution $p(c|\x_i)$ that we aim to model. We approximate $p(c|\x_i)$ using the parametric probability distribution $\sigma(c|\x,\theta)$,
	\begin{align}
		{\rm div}_{\rm kl}(p \| \sigma ) 
		&\simeq \sum_{i=1}^n\Bigl\{-\ln \sigma(c|\x_i,\theta)+\ln p(c|\x_i)\Bigr\}\\
		&=-\bigl\{\ln \sigma(\mT|\mX,\theta)-\ln p(\mT|\mX)\}\geq0\label{prml01119}\\
		{\rm H}(\mT|\mX)
		&\leq -\ln \sigma(\mT|\mX,\theta)\,.
		\label{upper}
	\end{align}
	where ${\rm div}_{\rm kl}$ denotes the Kullback-Leibler divergence. The left-hand side of Eq.~\eqref{upper} is independent of $\theta$, and the right-hand term is the negative log likelihood function for $\theta$ under the posterior distribution evaluated using the training dataset. Thus we see that minimizing Eq.~\eqref{prml01119} is equivalent to maximizing the likelihood function, \ie, minimizing this upper bound of the posterior entropy. For an i.i.d. finite dataset $\mX$\,, when approximating $p(c|\x)$ through a parametric distribution $\sigma(c|\x,\theta)$, the following Lemma~\ref{def1} holds: 
	\newtheorem{definition}{Lemma}
	\begin{definition}\label{def1}
		Maximizing the likelihood function is equivalent to minimizing the upper bound on the posterior entropy.
	\end{definition}
	\subsection{Data Entropy Neural Estimation}
	We approach ${\rm H}(\mX)$ by maximizing the mutual information of pairwise augmentation views of $\mX$~\cite{oord2018representation,belghazi2018mutual,ENSmymm2023}. Then, we have
	\begin{align}
		{\rm H}(\mX)={\rm I}(\mX; \mX) \geq {\rm I}(\mX^{(1)}; \mX^{(2)})\,.\label{lower}
	\end{align}
	Then, the mutual information between the two augmented views $\mX^{(1)}$ and $\mX^{(2)}$ is given by
	\begin{align}\label{MI}
		{\rm I}(\mX^{(1)};\mX^{(2)})
		={\rm div}_{\rm kl}\bigl(p(\x^{(1)},\x^{(2)})\|p(\x^{(1)})p(\x^{(2)})\bigr)
	\end{align}
	where $\x^{(1)}$ and $\x^{(2)}$ are two augmented data points.
	
	Following contrastive learning~\cite{oord2018representation,belghazi2018mutual,DSSN2023}, the maximization of the mutual information between two views turns into a lower bound maximization problem. Then, the following Lemma~\ref{def2} holds,
	\begin{definition}\label{def2}
		Maximizing the mutual information between two augmentation views is equivalent to maximizing the lower bound of the entropy.
	\end{definition}
	\subsection{Entropy Estimation for SSL}
	We assume that both $\mX_l$ and $\mX_u$ share the same distribution as $\mX$, and implement an encoder to optimize it from two bounds.
	
	Given the modeling of labeled data point $(\x,\gt)$ by the parameter $\theta$ to produce $\y=\sigma(c_{j}|\x_i,\theta)$, we encode all labeled data $\mX_l$, and the corresponding labeled data set is $\mT_l$. During the training process, we strive to align the predictive coding of \textit{InfoMatch} with the ground-truth coding, which essentially means modeling $y_{ij}$ to align closely with $p_{ij}$. In other words, we aim to approximate the entropy of the ground-truth posterior $-\ln p(\mT|\mX)$ with its negative log likelihood $-\ln \sigma(\mT|\mX,\theta)$. Referring to Lemma~\ref{def1}, the likelihood is
	\begin{equation}
		\sigma(\mT_l|\mX_l,\theta)=\prod_{i=1}^{n_l}\prod_{j=1}^k\sigma(c_j|\x_i,\theta) ^{p_{ij}}=\prod_{i=1}^{n_l}\prod_{j=1}^ky_{ij}^{p_{ij}}\,,
	\end{equation}
	where $p_{ij}$ is an element of $n_l \times k$ matrix of ground-truth matrix $P=[p_{ij}]$. The loss is defined by taking the negative logarithm of the likelihood, resulting in the cross-entropy loss:
	\begin{equation}
		\ell_{\rm upper}^l=-\ln \sigma(\mT_l|\mX_l,\theta)=-\sum_{i=1}^{n_l}\sum_{j=1}^kp_{ij}\ln y_{ij}\,.\label{loss_ce}
	\end{equation}
	
	Referring to Lemma~\ref{def2}, we employ the strategy of contrastive learning by maximizing the lower bound of entropy for $\mX_u$\,. We choose Jensen-Shannon divergence over Kullback-Leibler divergence, following~\cite{nowozin2016f}. The lower bound $\mathcal L_{\rm{lower}}$ is derived from
	\begin{align}
		&{\rm div}_{\rm js}\bigl(p(\x^{(1)},\x^{(2)})\|p(\x^{(1)})p(\x^{(2)})\bigr)\notag\\
		\geq&\mE_{p(\x^{(1)},\x^{(2)})}\log\bigl( d(\z^{(1)},\z^{(2)}|\theta)\bigr)\notag\\
		+&\mE_{p(\x^{(1)})p(\x^{(2)})} \log\bigl(1-d(\z^{(1)},\z^{(2)}|\theta)\bigr)=-\mathcal L_{\rm{lower}}\label{jsd}
	\end{align}
	where ${\rm div}_{\rm js}$ represents the Jensen-Shannon divergence and $d(\cdot,\cdot)$ is the similarity score of pairwise logits.
	
	Similar to~\cite{hjelm2018learning,DSSN2023}, we employ the view-wise contrastive loss, \ie, $\mathcal L_{\rm lower}$,
	\begin{align}
		\mathcal L_{\rm lower}
		=&- \frac1{|\mathcal P|}\sum_{(i,i)\in\mathcal P} \log d(\z_i^{(1)},\z_i^{(2)}|\theta)\notag\\
		&-\frac1{|\mathcal N|}\sum_{{(i,j)\in\mathcal N}} \log\bigl(1-d(\z_i^{(1)},\z_j^{(2)}|\theta)\bigr)\label{loss_cl}
	\end{align}
	where $\mathcal P$ and $\mathcal N$ denote positive and negative sets, respectively, \ie, $\z_i^{(1)}$ and $\z_i^{(2)}$ are belong to positive pairs $(i,i)\in\mathcal P$ while $\z_i^{(1)}$ and $\z_j^{(2)}$ are negative pairs $(i,j)\in\mathcal N, \forall\,i\neq j$. 
	\section{InfoMatch}
	Utilizing a neural network $\sigma(c_{j}|\x_i,\theta)$ to approximate the information entropy for semi-supervised image classification, we call our method \textit{InfoMatch}. For each unlabeled data point $\x\in\mD_u$, we conduct augmentation processing, creating one weak view, \ie, performing random flip, and two strong augmented views~\cite{cubuk2020randaugment} represented by $\x'$, $\x^{(1)}$, and $\x^{(2)}$, respectively.
	\subsection{Minimize Upper Bound} 
	In \textit{InfoMatch}, neural networks aim to approximate the upper bound of posterior entropy by maximizing the likelihood.
	
	According to Eq.~\eqref{loss_ce}, minimizing the cross entropy between the predicted-probability distribution and the ground-truth distribution amounts to minimizing the likelihood function. By employing optimization algorithms like gradient descent, \textit{InfoMatch} is gradually steered away from the incorrect distribution towards the accurate one.
	
	In practice, when dealing with labeled data, standard supervised learning which relies on cross-entropy loss is employed, and the corresponding loss function is designated as $\mathcal L^{l}_{\rm upper}$.
	
	To incorporate unlabeled data into training, we utilize pseudo supervision and weak-to-strong strategies. Let the one-hot vector $\hat{\gt}_i=[\hat{p}_{ij}]$ represent the pseudolabel corresponding to the weak-view prediction $\y'_i$, the pseudo-supervised loss function from weak to strong is given by:
	\begin{align}\label{L^u_upper}
		\mathcal L^{u}_{\rm upper}
		=-\frac12\sum_{i=1}^{n_u}\sum_{j=1}^km_{ij}\hat{p}_{ij}(\ln y^{(1)}_{ij}+\ln y^{(2)}_{ij})
	\end{align}
	where $\y^{(1)}_i$ and $\y^{(2)}_i$ correspond to the predictions for $\x^{(1)}_i$ and $\x^{(2)}_i$, respectively, and ${M}=[m_{ij}]$ is utilized to mask pseudolabels with confidence levels below a threshold.
	
	Additionally, we use CutMix to generate a new strong augmentation view on the top of weak augmentation, aiming to create more challenging samples and compel \textit{InfoMatch} to fully extract meaningful features. To enhance computational efficiency, we shuffle the dataset in batches and randomly select an image, denoted as $\x'_r$, at the corresponding position of $\x'_i$. Subsequently, these images are used to generate CutMix images in a one-to-one manner. Furthermore, we evaluate the quantity of semantic information retained from the original images within the CutMix images, according to the size of the region. Consequently, the CutMix image and its corresponding pseudolabel is expressed by:
	\begin{align}\label{CutMix}
		\x^{c}_{i} 
		&= \bm{b}_{\eta} \odot \x'_i + (\bm{1}-\bm{b}_{\eta}) \odot \x'_r\\
		\hat{p}^c_{ij}&=\eta m_{ij}\hat{p}_{ij}+
		(1-\eta)m_{rj}\hat{p}_{rj}
	\end{align}
	where $\bm{b}_{\eta}$ represents a random binary mask indicating where to exclude and incorporate information from two images, $\eta$ denotes the area proportion of $\x'_i$ in the mixed image, obtained by averaging the values of the binary matrix $\bm{b}_{\eta}$, and $\odot$ performs element-wise multiplication. The loss for the CutMix image is given by
	\begin{equation}\label{L^c_upper}
		\mathcal L^{c}_{\rm upper} = -\sum_{i=1}^{n_u}\sum_{j=1}^k\hat{p}^c_{ij}\ln y^{c}_{ij}
	\end{equation}
	where the prediction of $\x^{c}_{i}$ is $\y^{c}_{i}$. By combining aforementioned losses, we obtain the upper bound of posterior entropy,
	\begin{align}
		\mathcal L_{\rm upper}
		=\mathcal L^l_{\rm upper}+ \mathcal L^u_{\rm upper}+ \mathcal L^c_{\rm upper}\,.\label{upall}
	\end{align}
	\subsection{Maximize Lower Bound}
	\textit{InfoMatch} effectively captures the mutual information between pairs of augmentation views derived from the original dataset. By Eq.~\eqref{loss_cl}, minimizing the contrast loss is synonymous with maximizing the mutual information between pairwise augmentation views, thereby compelling \textit{InfoMatch} to delve into the inherent structure dictated by ${\rm H}(\mX)$. After multiple random augmentations, the distribution of augmentation views will gradually align with the distribution of the raw dataset, \ie, through multiple trainings with random augmentations of diverse samples, the neural network progressively approaches the information entropy of the original dataset.
	
	Inspired by BYOL~\cite{grill2020bootstrap}, we compute the contrastive loss by utilizing the positive pairs only, guaranteeing that samples of the same class are situated adjacent to each other in the embedding space. Furthermore, latent features often serve as covert structures or patterns derived from the data, offering a deeper understanding of the relationships within the data than merely predicting probability distributions. Consequently, we utilize the mutual information between the latent features corresponding to the two augmented views to approximate the lower bound of entropy.
	
	The similarity $d(\cdot,\cdot)$ of positive logits is defined by the Gaussian function~\cite{DSSN2023}, 
	\begin{align}\label{eq-gaussian}
		d(\z^{(1)}_i,\z_i^{(2)}|\theta)=\exp\Bigl(-{\bigl\|\z^{(1)}_i-\z_i^{(2)}\bigr\|^2_2}\Bigr)\,,
	\end{align}
	and the similarity defined in Eq.~\eqref{eq-gaussian} implies that if any two logits are identical, the similarity is one; otherwise, it tends to zero as their distance increases significantly.
	
	Substituting Eq.~\eqref{eq-gaussian} into Eq.~\eqref{loss_cl}, we obtain 
	\begin{align}
		\mathcal L_{\rm lower}
		=\frac1{n_u}\sum_{i=1}^{n_u}\bigl\|\z^{(1)}_i-\z_i^{(2)}\bigr\|^2_2\,.
		\label{loss_cll}
	\end{align}
	\subsection{InfoMatch Algorithm}
	We introduce an innovative semi-supervised image classification algorithm, called \textit{InfoMatch}, which treats the classification task as an entropy approximation problem. For labeled data, we utilize the cross entropy loss function as an upper bound for posterior entropy. \textit{InfoMatch} mainly focuses on effectively leveraging vast amounts of unlabeled data.
	
	In \textit{InfoMatch} algorithm, we initiate with pseudo supervision and weak-to-strong strategies. This allows us to transform the entropy upper bound into pseudo-supervised cross entropy loss, serving as a supervisory signal for unlabeled data. Furthermore, to enrich the diversity of our dataset, we introduce a new strong augmentation method, CutMix, that enhances its generalization capabilities by formulating corresponding loss. Subsequently, we consider the potential feature contrastive loss between the two augmented views as a lower bound for dataset entropy. By maximizing this lower bound and minimizing the upper bound, we gradually optimize \textit{InfoMatch}, even when limited labeled data is available.
	
	\begin{algorithm}[tp]
		\renewcommand{\algorithmicrequire}{\textbf{Input:}}
		\renewcommand{\algorithmicensure}{\textbf{Output:}}
		\caption{The \textit{InfoMatch} algorithm.}
		\label{algo1}	
		\begin{algorithmic}[1]
			\REQUIRE $\mD = \mD_l \cup \mD_u$, batch size $n^b_l$ and $n^b_u$, parameter $\lambda$.
			\ENSURE Optimal model parameters $\theta^\star$.
			\STATE \textbf{Initialization:} $epoch=0$, $epoch_{\max}$, and $\theta$.
			\WHILE {$epoch \le epoch_{\max}$}
			\FOR{mini-batch samples in $\mD$}
			\STATE Obtain augmented images $\x'$, $\x^{(1)}$, $\x^{(2)}$ and $\x^c$;
			\STATE Feed samples into model for forward propagation;
			\STATE Calculate $\mathcal L^l_{\rm upper}$ by Eq.~\eqref{loss_ce};
			\STATE Obtain the mask ${M}$ for filtering pseudolabels;
			\STATE Calculate $\mathcal L^u_{\rm upper}$ and $\mathcal L^c_{\rm upper}$ by Eqs.~\eqref{L^u_upper} and \eqref{L^c_upper};
			\STATE Calculate $\mathcal L_{\rm upper}
			=\mathcal L^l_{\rm upper}+ \mathcal L^u_{\rm upper}+ \mathcal L^c_{\rm upper}$
			\STATE Calculate $\mathcal L_{\rm lower}$ by Eq.~\eqref{loss_cll};
			\STATE Calculate $\mathcal L
			=\mathcal L_{\rm upper}+\lambda \mathcal L_{\rm lower}$
			\STATE Update parameter $\theta$ by minimizing $\mathcal L$;
			\STATE $epoch=epoch+1$;
			\ENDFOR
			\ENDWHILE
		\end{algorithmic} 
	\end{algorithm}
	
	The overall objective function incorporates the three cross-entropy losses of Eq.~\eqref{upall} and contrastive loss of Eq.~\eqref{loss_cll},
	\begin{align}
		\mathcal L
		=\mathcal L_{\rm upper}+\lambda \mathcal L_{\rm lower}
		\label{oveallloss}
	\end{align}
	where $\lambda$ serves as a non-negative hyperparameter, regulating the relative influence of the upper and lower bounds. During the experiment, we employ mini-batches for training, utilizing both labeled and unlabeled samples with batch sizes of $n^b_l$ and $n^b_u$, respectively. Algorithm~\ref{algo1} offers a comprehensive overview of the learning process.
	
	\begin{table*}[!tp]
		\centering
		\begin{tabular*}{0.92\textwidth}{@{\extracolsep{\fill}\,}l|cccc|c} 
			\toprule
			Dataset & \multicolumn{4}{c|}{CIFAR-10} &  CIFAR-100   \\
			\midrule
			\#~Label   & 10 & 40 & 250 & 4000 & 10000 \\
			\midrule
			$\Pi$ Model \tiny{\cite{rasmus2015semi}} &  ${79.18}$ \tiny ${ \pm 1.11}$ & ${74.34}$ \tiny ${ \pm 1.76}$ & ${46.24}$ \tiny ${ \pm 1.29}$ & ${13.13}$ \tiny ${ \pm 0.59}$ & ${36.65}$ \tiny ${ \pm 0.00}$  \\
			Pseudo Label \tiny{\cite{lee2013pseudo}} & ${80.21}$ \tiny ${ \pm 0.55}$ & ${74.61}$ \tiny ${ \pm 0.26}$ & ${46.49}$ \tiny ${ \pm 2.20}$ & ${15.08}$ \tiny ${ \pm 0.19}$ & ${36.55}$ \tiny ${ \pm 0.24}$  \\
			VAT \tiny{\cite{miyato2018virtual}} & ${79.81}$ \tiny ${ \pm 1.17}$ & ${74.66}$ \tiny ${ \pm 2.12}$ & ${41.03}$ \tiny ${ \pm 1.79}$ & ${10.51}$ \tiny ${ \pm 0.12}$ & ${32.14}$ \tiny ${ \pm 0.19}$   \\
			MeanTeacher \tiny{\cite{tarvainen2017mean}} & ${76.37}$ \tiny ${ \pm 0.44}$ & ${70.09}$ \tiny ${ \pm 1.60}$ & ${37.46}$ \tiny ${ \pm 3.30}$ & ${8.10}$ \tiny ${ \pm 0.21}$ & ${31.75}$ \tiny ${ \pm 0.23}$  \\
			MixMatch \tiny{\cite{berthelot2019mixmatch}} & ${65.76}$ \tiny ${ \pm 7.06}$ & ${36.19}$ \tiny ${ \pm 6.48}$ & ${13.63}$ \tiny ${ \pm 0.59}$ & ${6.66}$ \tiny ${ \pm 0.26}$ & ${27.78}$ \tiny ${ \pm 0.29}$  \\
			ReMixMatch \tiny{\cite{berthelot2019remixmatch}} & ${20.77}$ \tiny ${ \pm 7.48}$ & ${9.88}$ \tiny ${ \pm 1.03}$ & ${6.30}$ \tiny ${ \pm 0.05}$ & ${4.84}$ \tiny ${ \pm 0.01}$ & $\underline{20.02}$ \tiny ${ \pm 0.27}$  \\
			UDA \tiny{\cite{xie2020unsupervised}} & ${34.53}$ \tiny ${ \pm 10.69}$ & ${10.62}$ \tiny ${ \pm 3.75}$ & ${5.16}$ \tiny ${ \pm 0.06}$ & ${4.29}$ \tiny ${ \pm 0.07}$ & ${22.49}$ \tiny ${ \pm 0.23}$  \\
			FixMatch \tiny{\cite{sohn2020FixMatch}} & ${24.79}$ \tiny ${ \pm 7.65}$ & ${7.47}$ \tiny ${ \pm 0.28}$ & ${4.86}$ \tiny ${ \pm 0.05}$ & ${4.21}$ \tiny ${ \pm 0.08}$ & ${22.20}$ \tiny ${ \pm 0.12}$   \\
			Dash \tiny{\cite{xu2021dash}} & ${27.28}$ \tiny ${ \pm 14.09}$ & ${8.93}$ \tiny ${ \pm 3.11}$ & ${5.16}$ \tiny ${ \pm 0.23}$ & ${4.36}$ \tiny ${ \pm 0.11}$ & ${21.88}$ \tiny ${ \pm 0.07}$  \\
			MPL \tiny{\cite{pham2021meta}} & ${23.55}$ \tiny ${ \pm 6.01}$ & ${6.62}$ \tiny ${ \pm 0.91}$ & ${5.76}$ \tiny ${ \pm 0.24}$ & ${4.55}$ \tiny ${ \pm 0.04}$ & ${21.74}$ \tiny ${ \pm 0.09}$  \\
			FlexMatch \tiny{\cite{zhang2021flexmatch}} & ${13.85}$ \tiny ${ \pm 12.04}$ & ${4.97}$ \tiny ${ \pm 0.06}$ & ${4.98}$ \tiny ${ \pm 0.09}$ & ${4.19}$ \tiny ${ \pm 0.01}$ & ${21.90}$ \tiny ${ \pm 0.15}$  \\
			FreeMatch \tiny{\cite{wang2022freematch}} & $\underline{8.07}$ \tiny ${ \pm 4.24}$ & $\underline{4.90}$ \tiny ${ \pm 0.04}$ & ${4.88}$ \tiny ${ \pm 0.18}$ & ${4.10}$ \tiny ${ \pm 0.02}$ & ${21.68}$ \tiny ${ \pm 0.03}$  \\
			CoMatch \tiny{\cite{li2021comatch}} & - & ${6.91}$ \tiny ${ \pm 1.39}$ & ${4.91}$ \tiny ${ \pm 0.33}$ & ${4.27}$ \tiny ${ \pm 0.12}$ & ${22.11}$ \tiny ${ \pm 0.22}$  \\
			SimMatch \tiny{\cite{zheng2022simmatch}}  & - & ${5.60}$ \tiny ${ \pm 1.37}$ & $\underline{4.84}$ \tiny ${ \pm 0.39}$ & $\underline{3.96}$ \tiny ${ \pm 0.01}$ & ${20.58}$ \tiny ${ \pm 0.11}$  \\
			SimMatchV2 \tiny{\cite{zheng2023simmatchv2}}  & - & $\underline{4.90}$ \tiny ${ \pm 0.16}$ & ${5.04}$ \tiny ${ \pm 0.09}$ & ${4.33}$ \tiny ${ \pm 0.16}$ & ${21.37}$ \tiny ${ \pm 0.20}$  \\
			\midrule
			\textit{InfoMatch} & $\bm{4.39}$ \tiny ${ \pm 0.22}$ & $\bm{4.22}$ \tiny ${ \pm 0.14}$ & $\bm{4.01}$ \tiny ${ \pm 0.07}$ & $\bm{3.29}$ \tiny ${ \pm 0.08}$ & $\bm{19.47}$ \tiny{$\pm $0.56}\\
			\midrule
			Fully-Supervised & \multicolumn{4}{c|}{${4.62}$ \tiny ${ \pm 0.05}$} & ${19.30}$ \tiny ${ \pm 0.09}$ \\
			\bottomrule
		\end{tabular*}
		\caption{Top-1 error rates $(\%)$ on CIFAR-10/100 datasets. \textbf{Bold} indicates the best result, while \underline{underline} indicates the second-best result.}\label{cifar}
	\end{table*}
	\section{Experimental Results}
	\subsection{Experimental Setup}
	We evaluate \textit{InfoMatch} on well-known benchmark datasets, including CIFAR-10/100~\cite{krizhevsky2009learning}, SVHN~\cite{netzer2011reading}, STL-10~\cite{coates2011analysis}, and ImageNet~\cite{deng2009imagenet}. Additionally, we conduct experiments using varying amounts of labeled data.
	
	To ensure a fair comparison, we follow the experimental setup as in FixMatch \cite{sohn2020FixMatch} and FreeMatch \cite{wang2022freematch}. Specifically, we employ standard stochastic gradient descent algorithm with cosine learning rate decay as the optimizer across all datasets, with an initial learning rate of 0.03 and a momentum of 0.9. For all experiments, we set the total number of iterations to $2^{20}$. \textit{InfoMatch} performance is then evaluated using the EMA with a parameter of 0.999. Additionally, for ImageNet, we maintain a batch size of 128 for both labeled and unlabeled samples, \ie, $n^b_l = n^b_u = 128$, and utilize the ResNet-50 architecture \cite{he2016deep}. While for other datasets, we adjust the batch sizes to $n^b_l = 64$ and $n^b_u = 448$, and employ the Wide ResNet variants, such as Wide ResNet-28-2 \cite{zagoruyko2016wide} and Wide ResNet-28-8 \cite{zhou2020time}.
	
	To evaluate the impact of the two distinct terms of the loss Eq.~\eqref{oveallloss} in \textit{InfoMatch}, we introduce a non-negative hyperparameter $\lambda$. Then, we ensure an equitable representation of both real and pseudolabels in the likelihood function and maintain equilibrium between RandAugment~\cite{cubuk2020randaugment} and CutMix~\cite{yun2019CutMix}. Subsequently, we adjust the parameter $\lambda$ that regulates the entropy bounds to $0.002$.
	
	In addition, we employ distinct threshold selection strategies to eliminate the impact of pseudolabel errors on \textit{InfoMatch}. For SVHN, we adopt a fixed threshold of 0.95. Conversely, for other datasets, we utilize the self-adaptive thresholding method proposed by FreeMatch, which measures the overall learning progress by utilizing the expectation of the highest confidence across all unlabeled data in the current batch as the global threshold and assess the class-specific learning status by calculating the average prediction probabilities corresponding to each class.
	
	Finally, to guarantee a precise and unbiased evaluation, we conduct multiple training sessions for each model using various random seeds and calculate the mean and standard deviation of the optimal accuracy achieved.
	\subsection{Main Results}
	We compare \textit{InfoMatch} with full-supervised learning method and a range of representative semi-supervised learning methods, including pseudolabel-based methods such as FlexMatch~\cite{zhang2021flexmatch} and FreeMatch~\cite{wang2022freematch}, as well as graph-based methods like 
	SimMatch\cite{zheng2022simmatch} and SimMatchV2~\cite{zheng2023simmatchv2}. The Top-1 error rates for CIFAR-10/100, SVHN, and STL-10 under various labeled data sizes are presented in Tables~\ref{cifar} and \ref{svhn}.
	
	Based on these results, it is evident that \textit{InfoMatch} exhibits superior performance across all benchmarks, enhancing performance by an average of $1.49\%$, $0.55\%$, $0.13\%$, and $2.12\%$ on the four datasets. In particular, \textit{InfoMatch} impressively lower the mean Top-1 error rate from $8.07\%$ to $4.39\%$ ($-3.68\%$) in the CIFAR-10 dataset, which contains just 10 labeled data. While in in the STL-10 dataset with 40 labeled data, the Top-1 error rate is reduces from $13.74\%$ to $9.86\%$ (
	$-3.88\%$). Additionally, on both CIFAR-10 and SVHN datasets, our approach achieves remarkable performance improvements, surpassing not only other semi-supervised baselines but also outperforming fully-supervised learning methods across all benchmarks. Furthermore, comparing to other methods, \textit{InfoMatch} exhibits a smaller standard deviation across multiple experiments with varying seeds, thereby highlighting the superior stability and robustness of \textit{InfoMatch}.
	
	It is worth noting that \textit{InfoMatch} significantly outperforms other methods when the number of labeled data is extremely limited. Especially in the CIFAR-10 dataset, even when only one labeled data is available for each class, the average Top-1 error rate reachs $4.39\%$, which is lower than that of fully-supervised learning methods ($-0.23\%$).
	
	We assess \textit{InfoMatch} on ImageNet to show its efficacy. Following the settings of FreeMatch, we select 100 labeled samples per class. Table~\ref{imagenet} illustrates a comparison of the average Top-1 and Top-5 errors across different models. With identical parameters, \textit{InfoMatch} has a significant improvement of $4.49\%$ in Top-1 accuracy and $2.86\%$ in Top-5 accuracy over FreeMatch.
	
	\begin{table*}[!tp]
		\centering
		\begin{tabular*}{0.92\textwidth}{@{\extracolsep{\fill}\,}l|ccc|cc} 
			\toprule
			Dataset &  \multicolumn{3}{c|}{SVHN}  &  \multicolumn{2}{c}{STL-10}  \\
			\midrule
			\#~Label & 40 & 250 & 1000 & 40 & 1000  \\
			\midrule
			$\Pi$ Model \tiny{\cite{rasmus2015semi}} &  ${67.48}$ \tiny ${ \pm 0.95}$ & ${13.30}$ \tiny ${ \pm 1.12}$ & ${7.16}$ \tiny ${ \pm 0.11}$ & ${74.31}$ \tiny ${ \pm 0.85}$ & ${32.78}$ \tiny ${ \pm 0.40}$  \\
			Pseudo Label \tiny{\cite{lee2013pseudo}} & ${64.61}$ \tiny ${ \pm 5.6}$ & ${15.59}$ \tiny ${ \pm 0.95}$ & ${9.40}$ \tiny ${ \pm 0.32}$ & ${74.68}$ \tiny ${ \pm 0.99}$ & ${32.64}$ \tiny ${ \pm 0.71}$  \\
			VAT \tiny{\cite{miyato2018virtual}} & ${74.75}$ \tiny ${ \pm 3.38}$ & ${4.33}$ \tiny ${ \pm 0.12}$ & ${4.11}$ \tiny ${ \pm 0.20}$ & ${74.74}$ \tiny ${ \pm 0.38}$ & ${37.95}$ \tiny ${ \pm 1.12}$  \\
			MeanTeacher \tiny{\cite{tarvainen2017mean}} & ${36.09}$ \tiny ${ \pm 3.98}$ & ${3.45}$ \tiny ${ \pm 0.03}$ & ${3.27}$ \tiny ${ \pm 0.05}$ & ${71.72}$ \tiny ${ \pm 1.45}$ & ${33.90}$ \tiny ${ \pm 1.37}$  \\
			MixMatch \tiny{\cite{berthelot2019mixmatch}} & ${30.60}$ \tiny ${ \pm 8.39}$ & ${4.56}$ \tiny ${ \pm 0.32}$ & ${3.69}$ \tiny ${ \pm 0.37}$ & ${54.93}$ \tiny ${ \pm 0.96}$ & ${21.70}$ \tiny ${ \pm 0.68}$  \\
			ReMixMatch \tiny{\cite{berthelot2019remixmatch}} & ${24.04}$ \tiny ${ \pm 9.13}$ & ${6.36}$ \tiny ${ \pm 0.22}$ & ${5.16}$ \tiny ${ \pm 0.31}$ & ${32.12}$ \tiny ${ \pm 6.24}$ & ${6.74}$ \tiny ${ \pm 0.14}$  \\
			UDA \tiny{\cite{xie2020unsupervised}} & ${5.12}$ \tiny ${ \pm 4.27}$ & $\underline{1.92}$ \tiny ${ \pm 0.05}$ & $\underline{1.89}$ \tiny ${ \pm 0.01}$ & ${37.42}$ \tiny ${ \pm 8.44}$ & ${6.64}$ \tiny ${ \pm 0.17}$  \\
			FixMatch \tiny{\cite{sohn2020FixMatch}} & ${3.81}$ \tiny ${ \pm 1.18}$ & ${2.02}$ \tiny ${ \pm 0.02}$ & ${1.96}$ \tiny ${ \pm 0.03}$ & ${35.97}$ \tiny ${ \pm 4.14}$ & ${6.25}$ \tiny ${ \pm 0.33}$   \\
			Dash \tiny{\cite{xu2021dash}} & ${2.19}$ \tiny ${ \pm 0.18}$ & ${2.04}$ \tiny ${ \pm 0.02}$ & ${1.97}$ \tiny ${ \pm 0.01}$ & ${34.52}$ \tiny ${ \pm 4.30}$ & ${6.39}$ \tiny ${ \pm 0.56}$  \\
			MPL \tiny{\cite{pham2021meta}} & ${9.33}$ \tiny ${ \pm 8.02}$ & ${2.29}$ \tiny ${ \pm 0.04}$ & ${2.28}$ \tiny ${ \pm 0.02}$ & ${35.76}$ \tiny ${ \pm 4.83}$ & ${6.66}$ \tiny ${ \pm 0.00}$  \\
			FlexMatch \tiny{\cite{zhang2021flexmatch}} & ${8.19}$ \tiny ${ \pm 3.20}$ & ${6.59}$ \tiny ${ \pm 2.29}$ & ${6.72}$ \tiny ${ \pm 0.30}$ & ${29.15}$ \tiny ${ \pm 4.16}$ & ${5.77}$ \tiny ${ \pm 0.18}$  \\
			FreeMatch \tiny{\cite{wang2022freematch}} & $\underline{1.97}$ \tiny ${ \pm 0.02}$ & ${1.97}$ \tiny ${ \pm 0.01}$ & ${1.96}$ \tiny ${ \pm 0.03}$ & ${15.56}$ \tiny ${ \pm 0.55}$ & $\underline{5.63}$ \tiny ${ \pm 0.15}$  \\
			CoMatch \tiny{\cite{li2021comatch}}  & ${8.20}$ \tiny ${ \pm 5.32}$ & ${2.16}$ \tiny ${ \pm 0.04}$ & ${2.01}$ \tiny ${ \pm 0.04}$ & $\underline{13.74}$ \tiny ${ \pm 4.20}$ & ${5.71}$ \tiny ${ \pm 0.08}$  \\
			SimMatch \tiny{\cite{zheng2022simmatch}} & ${7.60}$ \tiny ${ \pm 2.11}$ & ${2.48}$ \tiny ${ \pm 0.61}$ & ${2.05}$ \tiny ${ \pm 0.05}$ & ${16.98}$ \tiny ${ \pm 4.24}$ & ${5.74}$ \tiny ${ \pm 0.31}$  \\
			SimMatchV2 \tiny{\cite{zheng2023simmatchv2}} & ${7.92}$ \tiny ${ \pm 2.80}$ & ${2.92}$ \tiny ${ \pm 0.81}$ & ${2.85}$ \tiny ${ \pm 0.91}$ & ${15.85}$ \tiny ${ \pm 2.62}$ & ${5.65}$ \tiny ${ \pm 0.26}$  \\
			\midrule
			\textit{InfoMatch} & $\bm{1.84}$ \tiny ${ \pm 0.07}$ & $\bm{1.79}$ \tiny ${ \pm 0.01}$ & $\bm{1.75}$ \tiny ${ \pm 0.03}$ & $\bm{9.86}$ \tiny ${ \pm 1.13}$ & $\bm {5.27}$ \tiny ${ \pm 0.09}$\\
			\midrule
			Fully-Supervised & \multicolumn{3}{c|}{${2.13}$ \tiny ${ \pm 0.01}$} & \multicolumn{2}{c}{-} \\
			\bottomrule
		\end{tabular*}
		\caption{Top-1 error rates $(\%)$ on SVHN and STL-10 datasets. \textbf{Bold} indicates the best result, while \underline{underline} indicates the second-best result.}\label{svhn}
	\end{table*}
	\begin{table}[tp]
	\centering
	\begin{tabular*}{0.46\textwidth}{@{\extracolsep{\fill}\,}l|cc} 
		\toprule
		Method & Top-1 &  Top-5 \\
		\midrule
		FixMatch \tiny{\cite{sohn2020FixMatch}} & $43.66$ & $21.80$ \\
		FlexMatch \tiny{\cite{zhang2021flexmatch}} & $41.85$ & $19.48$\\
		FreeMatch \tiny{\cite{wang2022freematch}} & $40.57$ & $18.77$\\
		\midrule
		\textit{InfoMatch} & $\bm{36.21}$ & $\bm{15.91}$\\
		\bottomrule
	\end{tabular*}
	\caption{Error rates (\%) on ImageNet with 100 labels per class.}\label{imagenet}
\end{table}
	\subsection{Ablation Study}
	\textit{InfoMatch} primarily explores how to utilize of unlabeled data, where we use the weak-to-strong supervision based on two strong augmentation methods, RandAugment~\cite{cubuk2020randaugment} and CutMix~\cite{yun2019CutMix}, represented by $\mathcal{L}^u_{\rm upper}$ and $\mathcal{L}^c_{\rm upper}$ respectively, as well as the contrastive loss $\mathcal{L}_{\rm lower}$ between two strong augmented views. To evaluate the impact of these losses, we conduct in-depth ablation studies on the CIFAR-10 dataset with 40 labels. 
	
	To determine the effectiveness of each loss term, we conduct experiments by excluding some components. As shown in Fig.~\ref{fig:ablation}, Results indicate that eliminating certain losses lead to different levels of decline in model performance. This suggests that each component is important and complementary. 
	
	Comparing to RandAugment~\cite{cubuk2020randaugment}, CutMix~\cite{yun2019CutMix} employs a stronger augmentation by combining two images, addressing the limitations of low mask-based augmentation ratio. This approach enables \textit{InfoMatch} to more effectively comprehend and capture the essential features, thus enhancing the generalization performance of \textit{InfoMatch}. However, utilizing CutMix exclusively may distort the original features, potentially leading \textit{InfoMatch} to learning incorrect or distorted features during training, ultimately resulting in performance deterioration. 

	\begin{figure}[!t]
	\centering
	\includegraphics[width=0.94\linewidth]{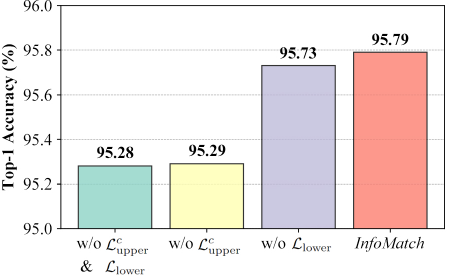}
	\caption{Ablation studies on the Cifar-10 dataset with 40 labels.}
	\label{fig:ablation}
\end{figure}
	By minimizing $\mathcal{L}_{\rm lower}$, we optimize \textit{InfoMatch} from a different perspective. \textit{InfoMatch} maximizes the mutual information between features in different augmented views, enabling \textit{InfoMatch} to learn feature representations that express the same information from different perspectives. Consequently, the performance of \textit{InfoMatch} is further improved.

		\begin{figure}[!t]
		\centering
		\includegraphics[width=0.48\textwidth]{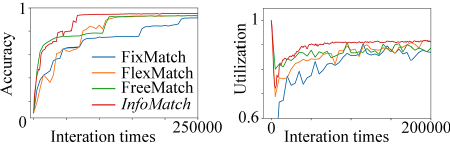}
		\caption{Compare \textit{InfoMatch} with FixMatch, FlexMatch, and FreeMatch on the CIFAR-10 dataset with 40 labeled data in terms of Top-1 accuracy and utilization of unlabeled data. (a) Top-1 accuracy, and (b) Utilization of unlabeled data.}
		\label{Analysis}
	\end{figure}
	\subsection{Qualitative Analysis}
	To further evaluate the performance of \textit{InfoMatch} in classification tasks, we conduct a comparative analysis with FixMatch, FlexMatch, and FreeMatch in terms of Top-1 accuracy and utilization of unlabeled data. This analysis is conducted using the CIFAR-10 dataset with 40 labeled data.
	
	Initially, a noteworthy advantage of \textit{InfoMatch} is its excellent convergence speed, which is illustrated in Fig.~\ref{Analysis}(a). Within 100 epochs, \textit{InfoMatch} achieves a $94.70\%$ accuracy rate. By the time it reaches 200 epochs, our accuracy even exceeds the highest accuracy achieved by FreeMatch, reaching $95.49\%$. Additionally, as shown in Fig.~\ref{Analysis}(b), in contrast to other methods, the utilization rate of unlabeled data in \textit{InfoMatch} quickly rises to and stabilizes above $90\%$ after a brief decline, further encouraging \textit{InfoMatch} to incorporate more unlabeled samples into training during the early stages. 
	\section{Conclusion}
	We present an extension of the entropy neural estimation method to the semi-supervised domain, enabling to measure the uncertainty of classification models from an entropy perspective. We introduce an upper bound on the entropy of the ground-truth posterior, indicating that maximizing the likelihood function of \textit{InfoMatch} corresponds to minimizing the upper bound on the posterior entropy. Similarly, we establish a lower bound on entropy for a given dataset, demonstrating that maximizing the mutual information between different augmented views of the data aligns with maximizing the lower bound of entropy. Drawing upon these two theories, we introduce \textit{InfoMatch} to optimize it from two perspectives. To leverage unlabeled data more effectively, we employ strategies like weak-to-strong and pseudo supervision, and introduce CutMix as a new strong augmentation view. Using optimization methods such as gradient descent, \textit{InfoMatch} ensures that its predicted probability distribution gradually aligns with the ground-truth distribution, even in the presence of input perturbations. 
	We introduce \textit{InfoMatch} that seamlessly integrates pseudo supervision, consistency regularization, and mixing strategies to improve its generalization capabilities. The results from various SSL benchmark tests demonstrate that \textit{InfoMatch} achieves good performance, especially in scenarios with limited labels, surpassing even fully-supervised learning methods. The experimental results confirm the effectiveness of \textit{InfoMatch}.
	
	\section*{Acknowledgements}
	This work was supported by the National Natural Science Foundation of China under Grant No.~62176108 , Natural Science Foundation of Qinghai Province of China under No.~2022-ZJ-929, Fundamental Research Funds for the Central Universities under No.~lzujbky-2022-ct06, and Supercomputing Center of Lanzhou University. The corresponding author of the paper is Kun Zhan.
	
	\bibliographystyle{named}
	\bibliography{ijcai24}
\end{document}